\documentclass{article}

\PassOptionsToPackage{numbers}{natbib}
\usepackage[preprint]{neurips_2019}

\usepackage[utf8]{inputenc} %
\usepackage[T1]{fontenc}    %
\usepackage{url}            %
\usepackage{booktabs}       %
\usepackage{amsfonts}       %
\usepackage{nicefrac}       %
\usepackage{microtype}      %

\usepackage{algorithm}
\usepackage[noend]{algorithmic}
\usepackage{amsmath,amssymb}
\usepackage{amsthm}
\usepackage[title]{appendix}
\usepackage{graphicx}

\usepackage{multirow, makecell}

\makeatletter
\newcommand{\algrule}[1][.2pt]{\par\vskip.2\baselineskip\hrule height #1\par\vskip.5\baselineskip}
\makeatother

\title{Learning Effective Loss Functions Efficiently}

\author{
  Matthew Streeter \\
  Google Research \\
  Mountain View, CA \\
  \texttt{mstreeter@google.com} \\
}

\newtheorem{theorem}{Theorem}

\newtheorem{innertheoremN}{Theorem}
\newenvironment{theoremN}[1]
  {\renewcommand\theinnertheoremN{#1}\innertheoremN}
  {\endinnertheoremN}

\newcommand{\argmin}{\mathrm{argmin}}

\newcommand{\ee}{\ensuremath{\mbox{ .}}}

\newcommand{\grad}{\nabla}
\newcommand{\ignore}[1]{}

\newcommand{\norm}[1]{\| #1 \| }
\newcommand{\reals}[0]{\mathbb{R}}

\newcommand{\set}[1]{\ensuremath{\left\{#1\right\}}}
\newcommand{\surl}[1]{\begin{small}\url{#1}\end{small}}
\newcommand{\tup}[1]{\langle#1\rangle}

\newenvironment{varalgorithm}[1]
  {\algorithm\renewcommand{\thealgorithm}{#1}}
  {\endalgorithm}

\newcommand{\concat}{\ensuremath{\frown}}
\newcommand{\conv}{\mathrm{convex\_pwl}}
\newcommand{\err}{e}  %
\newcommand{\feasible}{\mathcal{F}}

\renewcommand{\l}{\ell}  %
\renewcommand{\L}{\mathcal{L}}  %
\newcommand{\mfv}{\phi}  %

\newcommand{\ve}{\tilde e}  %
\newcommand{\w}{\theta}
\newcommand{\W}{\Theta}

\begin{document}

\maketitle

\begin{abstract}
We consider the problem of learning a loss function which,
when minimized over a training dataset, yields a model that approximately
minimizes a validation error metric.  Though learning an optimal loss
function is NP-hard, we present an anytime algorithm
that is asymptotically optimal in the worst case, and is
provably efficient in an idealized ``easy'' case.
Experimentally, we show that this algorithm can be used to tune
loss function hyperparameters orders of magnitude faster than
state-of-the-art alternatives.  We also show that our algorithm
can be used to learn novel and effective loss functions on-the-fly during
training.
\end{abstract}

\section{Introduction} \label {sec:intro}

Most machine learning models are obtained by minimizing a loss function,
but optimizing the training loss is rarely (if ever) the ultimate goal.
Instead, the model is judged based on its performance on test data not
seen during training, using a performance metric that may be only loosely
related to the training loss (e.g., top-1 error vs. log loss).
The ultimate value of a model therefore depends critically on the
loss function one chooses to minimize.

Despite the importance of choosing a good loss function,
it is unclear that the loss functions typically used in machine learning
are anywhere close to optimal.  For ImageNet classification, for example,
state-of-the-art models minimize log loss over the training
data, but the models are evaluated in terms of top-1 or top-5 accuracy.  Could
some other loss function lead to better results for these metrics?

In this work we seek to learn a loss function that, when (approximately)
minimized over the training data, produces a model that performs well on
test data according to some error metric.  The error metric need not be
differentiable, and may be only loosely related to the loss function.

Building on recent work on learning regularizers \cite{streeter2019learning},
we present a convex-programming-based algorithm that takes as input
observed data from training a small number of models, and produces as output
a loss function.  This algorithm can be used to tune loss function
hyperparameters, or to adjust the loss function
on-the-fly during training.  The algorithm comes with appealing theoretical
guarantees, and performs very well in our experiments.

Importantly, in contrast to previous work \cite{streeter2019learning},
our algorithm can make use of gradient information in the case where
the error metric is differentiable (or can be approximated by a differentiable
proxy function).  As we will show, using gradient information can
dramatically accelerate the search for a good loss function, and allows
us to efficiently discover loss functions with hundreds of hyperparameters
on-the-fly during training.

\subsection {Problem Statement} \label{sec:problem_statement}

We consider a general learning problem where the goal is to produce a
model from some set $\W \subseteq \reals^n$ of models, so as
to minimize a \emph{test error} $\err: \W \rightarrow \reals_{\ge 0}$.
Our model is obtained by minimizing a training loss
$\l: \W \rightarrow \reals_{\ge 0}$, which belongs to a set $\L$ of possible
loss functions.  We would like to find the $\l \in \L$ that, when minimized,
produces the lowest test error.  That is, we wish to solve the
bilevel minimization problem:
\begin{equation} \label{eq:bilevel}
  \min_{\l \in \L} \set { \err(\hat \w(\l)) }
  \mbox { where }
  \hat \w(\l) \equiv \argmin_{\w \in \W} \set {\l(\w)} \ee
\end{equation}
We assume that for any loss function $\l \in \L$, we can (approximately)
minimize $\l$ to
obtain $\hat \w(\l)$,
and that for any model $\w \in \W$, we can compute a
\emph{validation error} $\ve(\w)$, which is an estimate of test error.
In some cases, we may also be able to compute the gradient of validation error,
$\grad \ve(\w)$.

We will consider the case in which $\L$ is the set of linear functions of some user-provided, problem-specific feature vector $\mfv: \W \rightarrow \reals^k$.  Specifically, for a given feasible set $\feasible \subseteq \reals^k$, we assume
\[
        \L = \set{\l_\lambda\ |\ \lambda \in \feasible}\mbox{, where } \l_\lambda(\w) \equiv \lambda \cdot \mfv(\w) \ee
\]
Our goal is therefore to find the $\lambda \in \feasible$ that minimizes
\eqref{eq:bilevel}.

\subsection{Applications} \label {sec:applications}

\newcommand{\logloss}{\mathrm{logloss}}
\newcommand{\lex}{\l_{\mathrm{example}}}
\newcommand{\relu}{\mathrm{relu}}

The problem of learning an optimal linear loss function has many applications.  Perhaps the most obvious application is tuning loss function hyperparameters.  As an example, suppose we wish to do softmax regression with L1 and L2 regularization.  Our loss function is of the form:
\begin{equation} \label{eq:regularized_softmax}
        \l(\w) = \lambda_1 \norm{\w}_1 + \lambda_2 \norm{\w}_2^2 + \logloss(\w) \ee
\end{equation}
This loss function is linear with respect to the feature vector
$\mfv(\w) = \tup{\norm{\w}_1, \norm{\w}_2^2, \logloss(\w) }$.
Thus, finding an optimal loss function of the form $\lambda \cdot \mfv(\w)$,
optimizing over the feasible set $\feasible = \set{\lambda \in \reals^3_{\ge 0}\ |\ \lambda_3 = 1}$, will give us the optimal values of the
hyperparameters $\lambda_1$ and $\lambda_2$.

As a second example, suppose we wish to train an ImageNet classifier using data augmentation.  Given a set of $k$ possible image transformations (e.g., flipping horizontally, converting to grayscale), we apply a transformation drawn
randomly from some distribution whenever we train on an image.
The expected loss is of the form:
\begin{equation}
        \textstyle{\l(\w) = \sum_{j=1}^k p_j \l_j(\w)}
\end{equation}
where $\l_j$ is the log loss on a version of the ImageNet training set to
which transformation $j$ has been applied.  Finding an optimal
probability distribution is equivalent to finding an optimal loss of
the form $\lambda \cdot \tup{\l_1, \l_2, \ldots, \l_k}$ (which we can scale by
a $\frac {1} {\norm{\lambda}_1}$ factor to convert to the desired form).

As a final example, suppose we again wish to do softmax regression, but rather
than assuming a regularizer of a specific form (e.g., L1 or L2)
we wish to use a \emph{learned} convex function $r$:
\begin{equation} \label{eq:convex_regularizer}
        \textstyle{\l(\w) = \logloss(\w) + \sum_{i=1}^n r(\w_i)} \ee
\end{equation}
To find an approximately optimal loss function of this
form, we may require $r \in \conv(X)$, where $\conv(X)$ is the set of convex,
piecewise-linear functions that change slope at a predefined, finite set $X \subset \reals$ of points.
It can be shown that
this is equivalent to
the set of non-negative linear combinations of the functions
$\set{f_{\sigma, a}\ |\ \sigma \in \set{-1, 1}, a \in X}$,
where $f_{\sigma, a}(x) \equiv \max\set{0, \sigma(x-a)}$.
Using this fact, we can
write $\l$ as a linear function of a feature vector of length $1 + 2|X|$, whose first component
is $\logloss(\w)$, and whose remaining components are of the form
$\sum_{i=1}^n f_{\sigma, a}(\w_i)$.  By learning a linear loss function of this
form, we can discover novel, problem-specific regularizers.

\subsection{Summary of Results} \label {sec:summary}

We first consider the computational complexity of computing an optimal
linear loss function.  We find that:
\begin{itemize}
\item Computing an optimal linear loss function is NP-hard,
	even under strong assumptions about the set of models
		$\W$, the validation error $\ve$, and the feature vector
		$\mfv$.
\item However, if $\W$ is finite, an optimal loss function can be computed in time polynomial in $|\W|$.%
\end{itemize}

These findings suggest that we might select a finite set $\W_0 \subset \W$ of
models, then compute (in time polynomial in $|\W_0|$) a loss function that
is optimal when minimized over $\W_0$ (rather than over all of $\W$).  One
might hope that if $\W_0$ is sufficiently ``representative'', such a loss function would
also give good results when minimized over all of $\W$.

How big does $\W_0$ have to be in practice?  We address this question both
theoretically and experimentally.
Theoretically, we show that in the special case where $\ve(\w) = \lambda^* \cdot \mfv(\w)$, we can recover $\lambda^*$ 
after computing $\ve(\w)$ and $\grad \ve(\w)$ for a
\emph{single} model $\w$.
Experimentally, we show:
\begin{itemize}
\item When used to tune loss function hyperparameters based on results of
	full training runs, our algorithm
can outperform state-of-the-art alternatives by multiple orders of magnitude.
\item By tuning the loss function online, we can achieve test error
competitive with the results of extensive hyperparameter tuning during
the course of a single training run.
\end{itemize}

\section{What Makes a Good Loss Function?}

A good loss function is one that we can (approximately) minimize, and one
whose argmin has low test error.
To ensure that the loss functions we consider can be approximately minimized,
we confine our attention to linear functions of a user-provided feature
vector (which
can be minimized efficiently if, for example, each component of the feature
vector is a convex function of $\w$).
How can we guarantee that the argmin of training loss has low test error?
Assume we have already trained a small set $\W_0$ of models, and estimated
the test error of each of them using a validation set.
Given this data, we would like to
produce a new loss function that, when minimized, yields a model with
better validation error than any model we have already trained.

Ideally, we would find a loss function $\l$ such that $\l(\w) = \err(\w)\ \forall \w \in \W$, where $\err(\w)$ is the test error.  Minimizing $\l$ would
then give the best possible test error.
With this in mind, we might attempt to find an $\l$ that
estimates validation error as accurately as possible, for example in terms of
mean squared error over all $\w \in \W_0$.  Unfortunately, the argmin
of such a loss function may be far from optimal.
Figure~\ref{fig:loss_functions} illustrates this point for a one-dimensional
model $\w$, where test error is a piecewise-linear function of $\w$, but
training loss is constrained to be a quadratic function of $\w$.

\newcommand{\wstarz}{\w^*_0}

\begin{figure} [h]
	\begin{center}
	\includegraphics[width=0.49\linewidth]{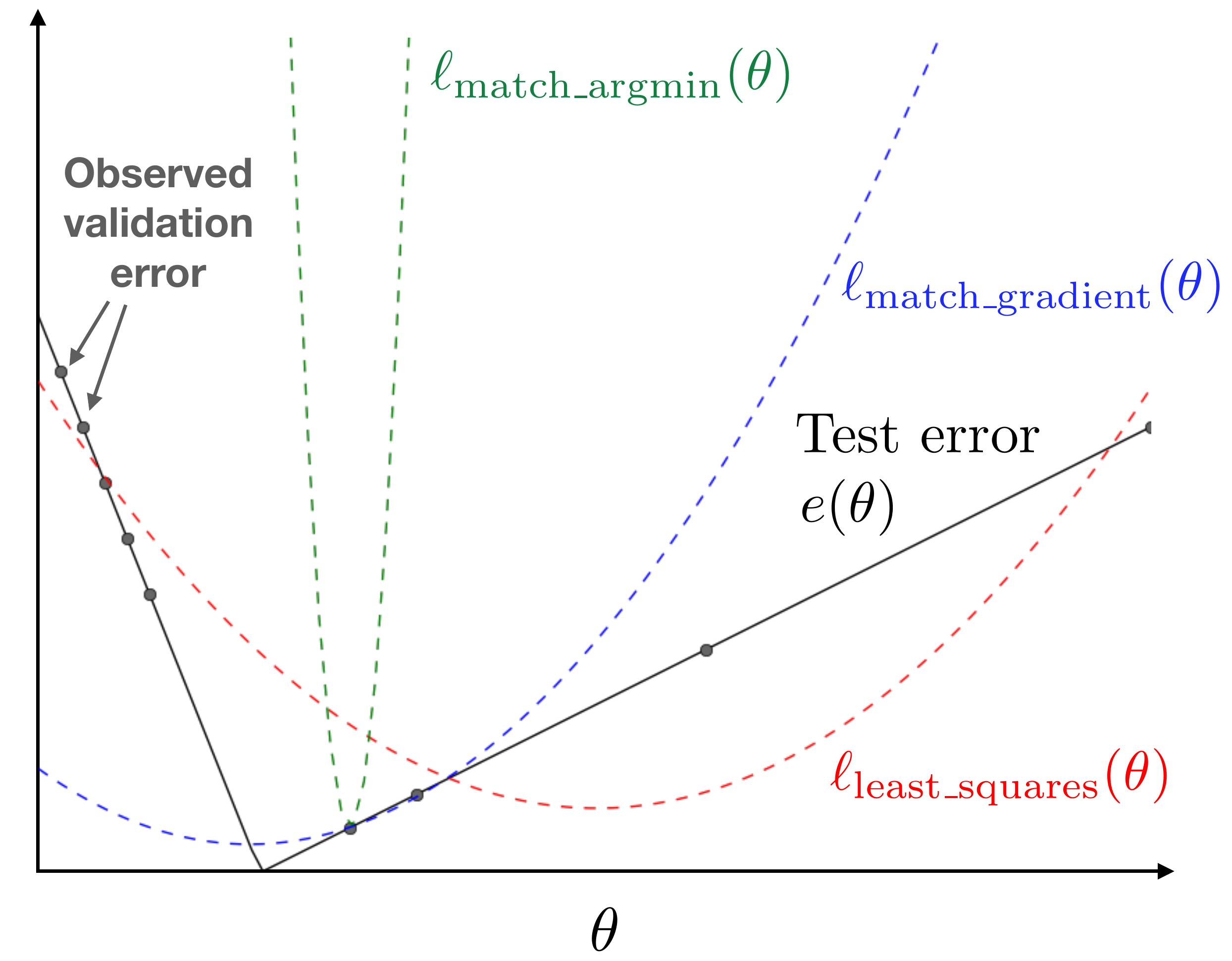}
	\caption{Comparison of three quadratic loss functions, on a one-dimensional minimization problem.}
\label{fig:loss_functions}
\end{center}
\end{figure}

To address this, we might seek a loss function that has the same argmin
as validation error when minmized over $\W_0$ (rather than over all $\W$).
Letting $\w^*_0$ be the model in $\W_0$ with least validation error,
we can easily construct such a function by setting
$\l(\w) = \norm{\w - \w^*_0}_2^2$.
However, this loss function is clearly not useful, because minimizing it gives
us back a model we have already seen.

Finally, we might seek a loss function that, in addition to having
the argmin-matching property, approximates
validation error well for models close to $\w^*_0$.  In the case where
$\ve$ is differentiable, this can be achieved by seeking a loss function $\l$
with $\grad \l(\w^*_0) \approx \grad \ve(\w^*_0)$, subject to the constraint
$\argmin_{\w \in \W_0} \set{ \l(\w)  } = \w^*_0$.  Minimizing such a loss
function often leads to a model with better validation (and test)
error, as illustrated in Figure~\ref{fig:loss_functions}.

\section{Learning Linear Loss Functions} \label{sec:learn}

We now present an algorithm for approximately solving the optimization
problem defined in \S\ref{sec:problem_statement}.  Recall that,
given
a set $\W \subseteq \reals^n$ of models,
a user-provided feature vector $\mfv: \W \rightarrow \reals^k$,
loss functions of the form $\l_\lambda(\w) = \lambda \cdot \mfv(\w)$,
and a feasible set $\feasible \subseteq \reals^k$ of $\lambda$ values,
our goal is to solve the bilevel minimization problem:
\begin{equation}\label{eq:bilevel_linear}
        \min_{\lambda \in \feasible} \set { \err(\hat \w(\l_\lambda)) }
  \mbox {, where }
	\hat \w(\l_\lambda) = \argmin_{\w \in \W} \set {\lambda \cdot \mfv(\w)}
	\ee
\end{equation}
Recall that $\err(\w)$ is the test error of $\w$, which we may estimate by
computing the validation error, $\ve(\w)$.
As discussed in \S\ref{sec:applications}, a solution to this problem has numerous practical applications, including tuning regularization
hyperparameters and learning data augmentation policies.

Ideally we would develop an algorithm that always recovers an optimal
$\lambda$ after evaluating validation loss a small (i.e., polynomial
in $n$ and $k$) number of times.
Unfortunately, doing so is NP-hard, even in the
special case when $\W$ is a convex set, $\mfv$ and $\err$ are convex
functions, and $\ve = \err$.
\begin{theorem} \label{thm:hardness}
Minimizing \eqref{eq:bilevel_linear} is NP-hard, even in the special case when
$\W = [0, 1]^n$, $\feasible = \reals^k$, $\mfv(\w) = \w$, $\ve(\w)$ is a convex function, and $\ve(\w) = \err(\w) \ \forall \w \in \W$.
\end{theorem}
\begin{proof}[Proof (Sketch)]
If $\mfv(\w) = \w$, then $\hat \w(\l_\lambda) = \argmin_{\lambda \in \reals^k} \set{\lambda \cdot \w}$.
Assuming ties are broken appropriately in cases where the argmin is not unique,
we have $\hat \w(\l_\lambda) \in \set{0,1}^n$.
Furthermore, $\set{\hat \w(\l_\lambda)\ |\ \lambda \in \reals^k} = \set{0, 1}^n$.
Minimizing \eqref{eq:bilevel_linear} is therefore equivalent to computing
$\min_{x \in \set{0, 1}^n} \set{\ve(x)}$, for an arbitrary convex function
$\ve$.  This optimization problem can be shown to be NP-hard, using a reduction
from {\sc 0/1 Integer Programming}.
\end{proof}
A formal proof of Theorem~\ref{thm:hardness} is given in Appendix A.

Though minimizing \eqref{eq:bilevel_linear} is NP-hard in general,
in the special case
where $\W$ is finite, it can be solved efficiently using a variant of the
LearnLinReg algorithm \cite{streeter2019learning}.
\begin{theorem} \label{thm:polytime}
If $\W$ is finite,
$\feasible$
is a hypercube,
and $\err = \ve$,
then \eqref{eq:bilevel_linear} can be minimized
in expected time $O(m d^{2.37} \log d)$, where $m = |\W|$
and $d = \max \set{m, |\mfv(\w)|}$,
assuming $\hat \w(\l_\lambda)$ is unique $\forall \lambda \in \feasible$.
\end{theorem}
\begin{proof}
Let the elements of $\W$ be indexed in ascending order of validation error,
so $\ve(\w_1) \le \ve(\w_2) \le \ldots \le \ve(\w_m)$.
If there exists a vector $\lambda \in \feasible$ such that $\hat \w(\l_\lambda) = \w_1$, then this $\lambda$ minimizes \eqref{eq:bilevel_linear}.
The constraint $\hat \w(\l_\lambda) = \w_1$ is equivalent to the system of
linear inequality constraints: $\lambda \cdot \mfv(\w_1) \le \lambda \cdot \mfv(\w_i)$ for $1 \le i \le m$.
Whether these constraints are satisfiable for some $\lambda \in \feasible$
can be determined using linear
programming, and the LP can be solved to machine precision
in time $O(d^{2.37} \log d)$ \cite{cohen2019solving}.

If the LP is feasible, any
feasible point is an optimal solution to \eqref{eq:bilevel_linear}.
If not, we can solve a similar LP to check whether there exists a
$\lambda \in \feasible$ that satisfies $\hat \w(\l_\lambda) = \w_2$, and so on, stopping
as soon as we find an LP that is feasible.
Because $\hat \w(\l_\lambda) \in \W$ for all $\lambda$,
at least one of the LPs must be feasible.
\end{proof}

\newcommand{\istar}{i^*}
\newcommand{\cost}{\mathrm{cost}}

Building on Theorem~\ref{thm:polytime}, we now present the \ref{alg:lll}
algorithm
for learning a linear loss function, given as input a small set $\W_0$
of models whose validation error is known.  The idea of the
algorithm is to use the ``guess the argmin'' trick
used in the proof of Theorem~\ref{thm:polytime}, to find a $\lambda$ that
would be optimal if the loss was minimized over $\W_0$ rather than $\W$
(i.e., if we replace $\W$ by $\W_0$ in \eqref{eq:bilevel_linear}).
However, in the common case where many such $\lambda$ exist,
\ref{alg:lll} returns the one that minimizes a carefully-chosen
cost function that encourages $\l_\lambda$ to accurately predict validation
error.  For some $\epsilon \ge 0$, we minimize
\begin{equation} \label{eq:cost}
  \cost(\lambda, \alpha) \equiv \sum_{ \w \in \W_0 } ( \l_\lambda(\w) - \alpha\ \ve(\w)  )^2 + \overbrace{\epsilon \norm{ \grad \l_\lambda(\w) - \alpha\ \grad \ve(\w)  }_2^2}^{\text{optional, if $\ve$ is differentiable}}
\end{equation}
where $\alpha > 0$ is a learned multiplier used to convert validation error
to an appropriate scale.

\ref{alg:lll} has two desirable theoretical guarantees.  First, by an argument
simliar to the one used to prove Theorem~\ref{thm:polytime}, it runs in
polynomial time
and returns a loss function that would be optimal if the loss was minimized
over $\W_0$ rather than over $\W$, as summarized in Theorem~\ref{thm:convergence}.
Second, it is provably efficient in certain special cases, as shown
in Theorem~\ref{thm:efficiency}.

\newcommand{\llgradi}[1]{g_{#1}}
\newcommand{\llfvi}[1]{\mfv_{#1}}
\newcommand{\llmodeli}[1]{\w_{#1}}
\newcommand{\lllossi}[1]{\l_{#1}}
\newcommand{\lljacobi}[1]{J_{#1}}
\newcommand{\llvei}[1]{\ve_{#1}}

\begin{varalgorithm}{LearnLoss}
  \begin{algorithmic}
  \caption{}
  \label{alg:lll}
  \STATE {\bfseries Input:} Set of (validation error, feature vector) pairs
$\set{(\llvei{i}, \llfvi{i}) \ |\ 1 \le i \le m}$,
feasible hypercube $\feasible \subseteq \reals^k$,
scalar $\epsilon \ge 0$.
  \STATE {\bfseries Optional input}: gradient vectors
  $\llgradi{i} \in \reals^n$, and Jacobian matrices
  $\lljacobi{i} \in \reals^{n \times k}$,
  for $1 \le i \le m$.
  Here
  $\llgradi{i} = \grad \ve(\llmodeli{i})$, and
  column $j$ of $\lljacobi{i}$ is $\grad \mfv_j(\llmodeli{i})$,
  where $\llmodeli{i}$ is the model for pair $(\llvei{i}, \llfvi{i})$.
  \algrule
  \STATE Sort $(\llvei{i}, \llfvi{i})$ pairs in ascending order of validation error, and reindex so $\ve_1 \le \ve_2 \le \ldots \le \ve_m$.
  \FOR {$\istar$ from $1$ to $m$}
    \STATE Solve the following convex quadratic program:
\begin{equation*}
\begin{array}{llll}
  \text{minimize}_{\lambda \in \feasible, \alpha \in \reals_{+}} & \sum_{i=1}^m (\lllossi{i} - \alpha\ \llvei{i})^2  + \overbrace{\epsilon \sum_{i=1}^m \norm{\lljacobi{i} \lambda^T - \alpha\ \llgradi{i}}_2^2}^{\text{if gradients were provided as input}} \\
  \text{subject to}
        & \lllossi{i} = \lambda \cdot \llfvi{i} & \forall i \\
  & \lllossi{i^*} \le \lllossi{i} & \forall i \\
\end{array}
\end{equation*}
  \STATE If the QP is feasible, return $\lambda$.
  \ENDFOR
\end{algorithmic}
\end{varalgorithm}

\begin{theorem} \label{thm:convergence}
Let $\W_0 \subseteq \W$ be a finite set of models.
Given as input the set of pairs $\set{(\ve(\w), \mfv(\w)) \ |\ \w \in \W_0}$, \ref{alg:lll} returns a vector $\hat \lambda \in \argmin_{\lambda \in \feasible} \set { \ve(\hat \w_0(\l_\lambda)) }$,
where $\hat \w_0(\l_\lambda) = \argmin_{\w \in \W_0} \set { \lambda \cdot \mfv(\w) }$.  It runs in time $O(m k^4)$, where $m = |\W_0|$ and $k = |\mfv(\w)|$.
\end{theorem}

If $\W$ is finite, Theorem~\ref{thm:convergence} shows that \ref{alg:lll} is
asymptotically optimal as $\W_0 \rightarrow \W$.  Under what circumstances is
\ref{alg:lll} efficient?  To build intuition, we consider the idealized
case where there exists a linear loss function that \emph{perfectly} estimates
validation error (and is therefore optimal if $\ve = \err$).
In this case, \ref{alg:lll} can recover this loss
function very efficiently, as shown in Theorem~\ref{thm:efficiency}.

\newcommand{\rank}{\mathrm{rank}}
\newcommand{\sloss}{S_\mathrm{loss}}
\newcommand{\sgrads}{S_\mathrm{grads}}
\begin{theorem} \label{thm:efficiency}
Suppose that for some $\lambda^* \in \feasible$ and $\alpha^* > 0$,
we have $\lambda^* \cdot \mfv(\w) = \alpha^*\ \ve(\w) \ \forall \w \in \W$.
Let
$\llvei{i} \in \reals$,
$\llfvi{i} \in \reals^k$,
$\llgradi{i} \in \reals^n$,
and $\lljacobi{i} \in \reals^{n \times k}$
be defined as in the code for \ref{alg:lll}.
Then, \ref{alg:lll} returns
$\lambda^*$ if at least $k+1$ vectors in the set
$(\sloss \cup \sgrads) \subset \reals^{k+1}$ are
linearly independent, where:
\begin{enumerate}
\item $\sloss \equiv \set{\llfvi{i} \concat \tup{\llvei{i}} |\  i \in [m]}$, where $\concat$ denotes concatenation, and
\item $\sgrads \equiv \set {(\lljacobi{i} \concat \llgradi{i})_j\ |\ i \in [m], j \in [n]}$ if
gradients are provided, otherwise $\sgrads \equiv \emptyset$.
\end{enumerate}
\end{theorem}
\begin{proof}
Because $\l_{\lambda^*}(\w) = \alpha^*\ \ve(\w)\ \forall \w \in \W$,
we have $\cost(\lambda^*, \alpha^*) = 0$.
Furthermore, because $\l_{\lambda^*}$ has the
same argmin as $\ve$, $\lambda^*$ is an optimal solution to the
first quadratic program considered by \ref{alg:lll} (i.e., the quadratic
program solved when $i^* = 0$).
Thus, \ref{alg:lll} will return the optimal vector $\lambda^*$, provided the
solution to the first quadratic program is unique.

In order to satisfy $\cost(\lambda, \alpha) = 0$, $\lambda$ and $\alpha$
must satisfy $m$ linear equations of the form
$\lambda \cdot \llfvi{i} - \alpha\ \llvei{i} = 0$.
If gradient information is provided,
$\lambda$ and $\alpha$ must, additionally, satisfy $m n$ equations of the
form $(\lljacobi{i} \lambda^T - \alpha \llgradi{i})_j = 0$.
This is a system of linear equations with $k + 1$ variables, and by assumption
at least $k+1$ of the equations are linearly independent, which guarantees a
unique solution.
\end{proof}

In particular, Theorem~\ref{thm:efficiency} shows that if a perfect loss
function exists, \ref{alg:lll} can recover it given $\ve(\w)$, $\mfv(\w)$, and
$\grad \ve(\w)$ for just \emph{one} model.
This is clearly a strong assumption that is unlikely to be literally
satisfied in practice.  Nevertheless, our experiments will show that
on certain real-world problems, \ref{alg:lll} achieves efficiency similar
to what Theorem~\ref{thm:efficiency} suggests.

\subsection{Tuning Loss Functions} \label{sec:tuneloss}

The \ref{alg:lll} algorithm suggests a natural iterative procedure for
tuning loss functions.
Let $\W_0$ be an initial
set of trained models (obtained, for example, as intermediate checkpoints
when minimizing an initial ``default'' loss function).
After computing the validation error of each $\w \in \W_0$, we run
\ref{alg:lll} to obtain a loss function $\l_1$.  We then minimize $\l_1$
to obtain a model, $\w_1$.  Computing the validation error of $\w_1$ then
provides an additional data point we can use to re-run \ref{alg:lll},
obtaining a refined loss function $\l_2$, and so on.
Pseudocode is given below.

\newcommand{\optimize}{\mathrm{train\_with\_warm\_start}}
\begin{varalgorithm}{TuneLoss}
\begin{algorithmic}
  \caption{}
  \label{alg:tl}
	\STATE {\bfseries Input:} validation error $\ve$,
	initial set of models $\W_0 \subseteq \W$, %
	feature vector function $\mfv: \W \rightarrow \reals^k$,
	initial warm-start model $\hat \w_0 \in \W_0$,
	feasible hypercube $\feasible$, scalar $\epsilon \ge 0$.
    \STATE Set $D_0 \leftarrow \set{(\ve(\w), \mfv(\w))\ |\ \w \in \Theta_0}$.
    \FOR {$i=1, 2, \ldots$}
	\STATE Set $\lambda_i \leftarrow \mathrm{LearnLoss}(D_{i-1}, \feasible, \epsilon)$.
	\STATE Set $\hat \w_i \leftarrow \optimize(\l_i, \hat \w_{i-1})$, where $\l_i(\w) \equiv \lambda_i \cdot \mfv(\w)$.
	\STATE Set $D_i \leftarrow D_{i-1} \cup \set{(\ve(\hat \w_i), \mfv(\hat \w_i))}$.
    \ENDFOR
\end{algorithmic}
\end{varalgorithm}

\ref{alg:tl} makes use of a subroutine, $\optimize$.  If this subroutine
runs an online algorithm such as AdaGrad \cite{duchi2011adaptive} for a
small number of mini-batches, then \ref{alg:tl} will adjust the loss function
online during training.  If the subroutine instead performs a full training
run (possibly ignoring the second argument), \ref{alg:tl} becomes a
sequential hyperparameter tuning algorithm.
\ref{alg:tl} can also be modified to provide the optional gradient
information in the calls to \ref{alg:lll}.

\section{Experiments} \label {sec:experiments}

We now apply \ref{alg:tl} to two problems discussed in \S\ref{sec:applications}: tuning loss function hyperparameters,
and learning novel convex regularizers on-the-fly during training.

\subsection{Methods} \label {sec:problems}

We consider image classification problems using four public datasets:
\emph{caltech101} \cite{feifei2004learning},
\emph{colorectal\_histology} \cite{kather2016multi},
\emph{oxford\_iiit\_pet} \cite{parkhi12a}, and
\emph{tf\_flowers} \cite{tfflowers}.
For each dataset, we train classifiers using transfer learning.
Starting with an Inception-v3
model trained on ImageNet \cite{russakovsky2015imagenet},
we adapt the model to classify images from the target dataset by
retraining the last layer of the network, as in \cite{donahue2014decaf}.
This approach yields strong, though not state-of-the-art, performance
on each problem.
Each dataset is split into
training, validation, and test sets as described in Appendix B.

We implemented \ref{alg:lll} in python, using
CVXPY \cite{diamond2016cvxpy} as the quadratic program solver,
and using AdaGrad \cite{duchi2011adaptive} for model training.

\subsection{Tuning Loss Function Hyperparameters} \label {sec:hyperparameter_tuning}

\newcommand{\luniform}{\l_{\mathrm{uniform}}}
\newcommand{\ldropout}{\l_{\mathrm{dropout}}}

We first consider using \ref{alg:tl} to tune the hyperparameters of a
hand-designed loss function.  Specifically, we use a loss function
with four regularization hyperparameters, of the form:
\[
  \l(\w) = \logloss(\w) + \lambda_1 \norm{\w}_1 + \lambda_2 \norm{\w}_2^2 + \lambda_3 \luniform(\w) + \lambda_4 \ldropout(\w)
\]
where $\luniform$ is the loss on a uniformly-labeled version of the training
dataset and $\ldropout$ is the loss using dropout with keep probability 0.5.
Training with this loss using SGD is equivalent to applying dropout to a given
example with probability $\frac {\lambda_4} {1 + \lambda_4}$,
and rescaling appropriately.

We compare \ref{alg:tl} to random search, Bayesian optimization
using GP-EI-MCMC \cite{snoek2012practical},
and the recent TuneReg algorithm \cite{streeter2019learning}.
All algorithms optimize performance on a validation set, and are evaluated
using a separate held-out test set.
While all four algorithms can optimize top-1 validation error directly, for \ref{alg:tl}
we instead optimize validation log loss, which allows us to take advantage of
gradient information.
All algorithms optimize over the same feasible set, defined in Appendix B.

\begin{figure} [h]
\begin{center}
\includegraphics[width=0.49\linewidth]{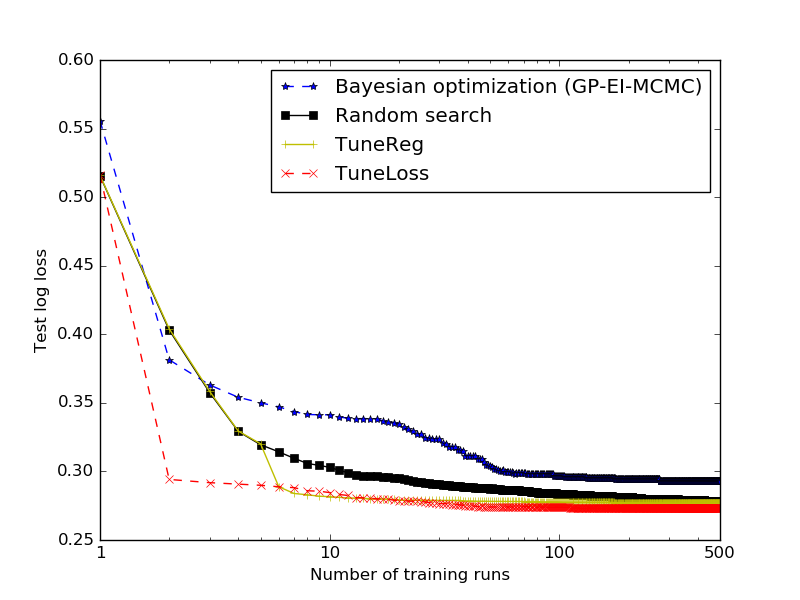}
\includegraphics[width=0.49\linewidth]{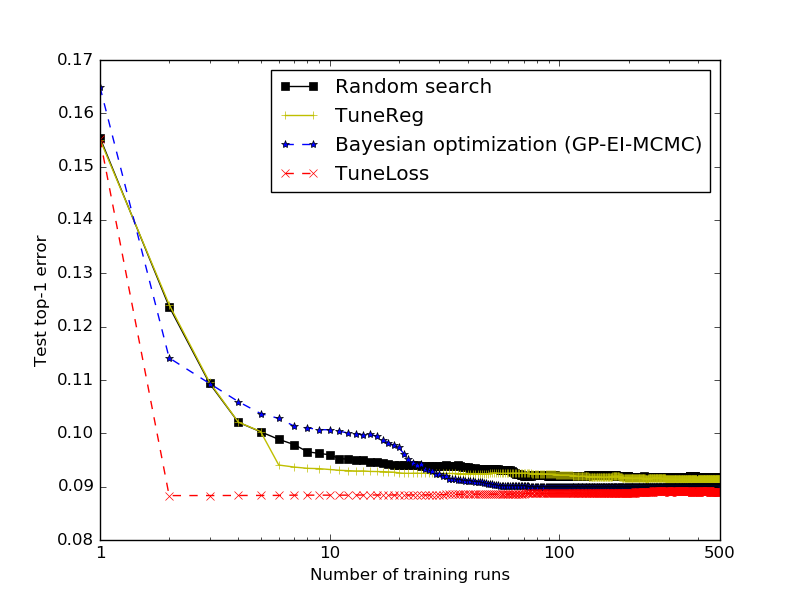}
\caption{Comparison of algorithms for tuning regularization hyperparameters.  Each curve is the average of 100 independent runs.  TuneLoss offers order of magnitude improvements in the number of training runs required to reach a given test loss or test error.}
\label{fig:hyperparameter_tuning}
\end{center}
\end{figure}

Figure~\ref{fig:hyperparameter_tuning} shows the test error and test loss
of each algorithm as a function of the number of training
runs performed, for the \emph{colorectal\_histology} dataset.
As is standard, we show the best test loss (resp.\ error) for the model with best-so-far
validation loss (resp.\ error).
Each curve is the average of 100 independent runs.  Observe that \ref{alg:tl}
offers {\bf order-of-magnitude} improvements in the number of training runs that must be performed to reach a given test error.  In particular, \emph{\ref{alg:tl} achieves better test error after 2 training runs than random search or GP-EI-MCMC achieve after 500 runs}.  We see similar improvements on all four datasets.

Table~\ref{tab:errors} shows the test error achieved by each
algorithm after 10 training runs, averaged over 100 runs of each algorithm.
\ref{alg:tl} reaches the lowest test error on all four datasets (italicized).

\begin{table*}[h]
	\caption{Test error for various hyperparameter tuning algorithms, after 10 training runs.}
  \label{tab:errors}
  \centering
	\begin{small}
	\begin{sc}
		\begin{tabular}{p{3.6cm}p{2cm}p{2cm}p{2.1cm}p{1.9cm}}
    \toprule
	  Dataset & TuneLoss (this paper) & Random search & GP-EI-MCMC \cite{snoek2012practical} & TuneReg \cite{streeter2019learning} \\
    \midrule

Caltech 101
& \emph{12.90\%} & 16.01\%& 18.00\%& 14.18\% \\

Colorectal histology
& \emph{8.84\%} & 9.60\%& 10.07\%& 9.32\% \\

Oxford IIIT Pet
& \emph{7.72\%} & 8.66\%& 9.11\%& 8.38\% \\

TF-Flowers
& \emph{9.57\%} & 10.63\%& 11.40\%& 10.03\% \\

\bottomrule
  \end{tabular}
  \end{sc}
  \end{small}
\end{table*}

\subsection{Learning Novel Regularizers Online} \label {sec:novel_regularizers}

Regularization is the subject of a vast literature.  In statistics,
a long line of research has focused on the functional form of the
regularizer \cite{frank1993statistical,fu1998penalized,hoerl1970ridge,tibshirani1996regression,zou2005regularization}.  In machine learning,
online algorithms such as AdaGrad \cite{duchi2011adaptive}
implicitly
use an adaptive proximal quadratic regularizer \cite{mcmahan2017survey},
but the regret-based analysis of such methods holds only
for a single pass over the training data.

\ref{alg:tl} has the potential to extend this work in two ways: (\emph{i})
it can \emph{learn} the functional form of the regularizer,
and (\emph{ii}) it can adapt the learned regularizer during training
to prevent overfitting,
even after passing over the training data many times.

To learn the functional form of the regularizer, we use TuneLoss
to learn a loss of the form:
\[
  \l(\w) = \logloss(\w) + \sum_{i=1}^n r(\w_i)\mbox{, where } r \in \conv(X)
\]
where $\conv(X)$ is the set of piecewise-linear convex functions of a single
variable, whose slope only changes at a predefined set $X$ of points.
As discussed in \ref{sec:applications}, $\l(\w)$ can be expressed as a linear
loss function using
a feature vector $\mfv(\w)$ of
length $1+2|X|$.
We learn the loss function online by having the
$\optimize$ subroutine
warm start from the latest checkpoint, and then perform one epoch of
AdaGrad, as discussed in \S\ref{sec:tuneloss}.
See Appendix B for additional details.

\begin{figure} [h]
  \begin{center}
    \includegraphics[width=0.49\linewidth]{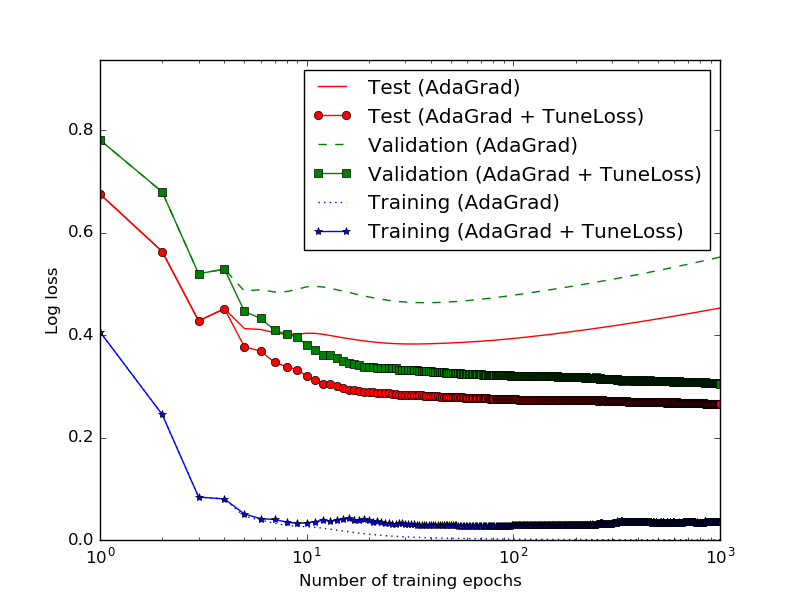}
    \includegraphics[width=0.49\linewidth]{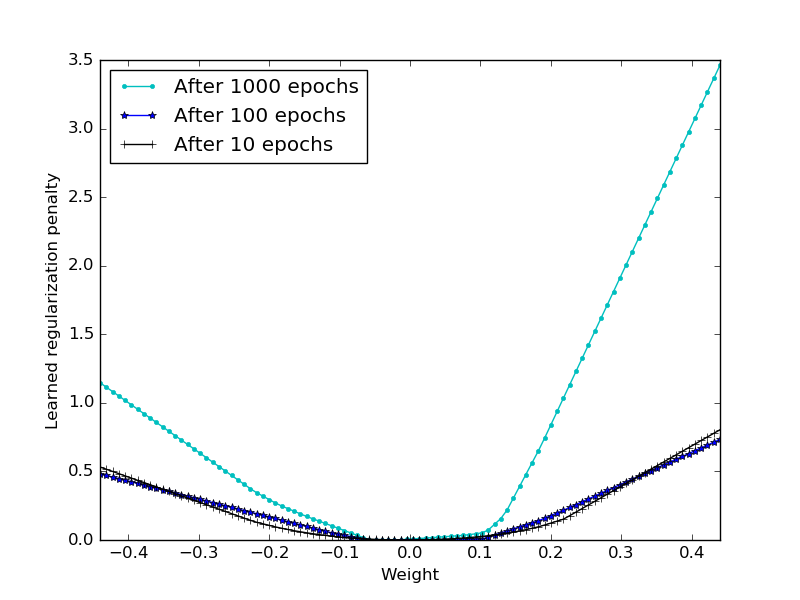}
    \caption{Evolution of training, validation, and test log loss when learning a regularizer online using TuneLoss (a), and the resulting
learned regularizers (b).  TuneLoss prevents overfitting,
and outperforms the result of extensive regularization hyperparameter tuning
(see Figure~\ref{fig:hyperparameter_tuning}).
}
  \label{fig:learned_regularizers}
\end{center}
\end{figure}

Figure~\ref{fig:learned_regularizers} shows results for the
\emph{colorectal\_histology}
dataset.  Plot (a) shows how how training, test, and
validation log loss change as a function of the number of epochs, compared to
a run that does not use regularization.
Plot (b) shows the regularizers that were learned after 10, 100, and 1000 epochs.  Observe that:
\begin{itemize}
  \item While AdaGrad starts overfitting after \textasciitilde35 epochs,
  AdaGrad + \ref{alg:tl} continues to improve log loss (both validation and test) even after training for 1000 epochs.
  \item The learned regularizer changes over time, with stronger regularization used later in training.
  \item During the course of a single training run, AdaGrad + \ref{alg:tl}
  reach test loss (and test error) better than that of
  the best linear combination of four regularizers (L1, L2, label smoothing, and  dropout) shown in Figure~\ref{fig:hyperparameter_tuning}.
\end{itemize}

We obtain similar improvements on the other three datasets.  In all cases,
validation and test log loss decrease monotonically when using \ref{alg:tl}.

These large improvements in log loss also lead to small but
not statistically significant improvements
in top-1 error.  Nevertheless, in applications where one
is interested in the actual values of the predicted probabilities (as is
likely the case for \emph{colorectal\_histology}), improvements in
log loss have significant practical benefit.

\section{Related Work} \label {sec:related}

Our work extends recent work by \citet{streeter2019learning}
on learning optimal linear regularizers.  In particular,
our \ref{alg:lll} algorithm simplifies and generalizes the LearnLinReg
algorithm of \cite{streeter2019learning}.  Critically,
\ref{alg:lll} can make use of gradient information
in order to learn an effective loss function more efficiently,
allowing us to take on problems with a much larger
number of hyperparameters (see \S\ref{sec:novel_regularizers}).

Our work shares the same goals as recent work on ``learning to teach''
\cite{fan2018learning,wu2018learning}, which uses a neural network
to learn a loss function during training.  Though this work showed
improvements in top-1 error, it is not clear how its efficiency compares to
that of existing hyperparameter tuning methods.

As a hyperparameter tuning method that makes use of gradients, \ref{alg:lll}
may at first appear similar to gradient-descent-based
methods \cite{bengio2000gradient,fan2018learning,pedregosa2016hyperparameter,wu2017bayesian}.  These methods differentiate validation loss
with respect to the \emph{hyperparameters}, which generally requires unrolling
the entire optimization process \cite{fan2018learning}
or approximating the gradients \cite{pedregosa2016hyperparameter}.
In contrast, \ref{alg:lll} differentiates
validation loss with respect to the \emph{model parameters}, and
learns hyperparameters that make these gradients match the gradients of the
learned loss function, a fundamentally different approach.

Finally,
a number of specialized algorithms have been developed for tuning regularization
hyperparameters.
Given a known data distribution,
\citet{liang2009asymptotically} provides a technique for deriving
a quadratic approximation to the expected test loss, which can then
be used to estimate optimal regularization hyperparameters.
More recently, \citet{mackay2019selftuning}
uses hypernetworks
to approximate the optimal model weights as a function of the regularization
hyperparameters.  Though promising, none of these techniques have been
shown to provide
improvements comparable to the ones shown in \S\ref{sec:hyperparameter_tuning}.

\section{Conclusions} \label {sec:conclusions}

Learning linear loss functions is a fundamental problem with many
interesting applications.  Though the problem is NP-hard, the \ref{alg:lll}
algorithm is provably efficient in an idealized easy case, and appears to
work well in practice.  In particular, this algorithm can be used to
(\emph{i}) solve certain hyperparameter tuning problems very efficiently, and
(\emph{ii}) prevent overfitting by learning an effective regularizer on-the-fly
during training.

\bibliographystyle{plainnat}
\bibliography{learnloss}

\begin{appendices}

\section{Proofs}

We now present a formal proof of Theorem~\ref{thm:hardness}.

\newcommand{\ind}{\mathbb{I}}

\begin{theoremN}{1}
Minimizing \eqref{eq:bilevel_linear} is NP-hard, even in the special case when
$\W = [0, 1]^n$, $\feasible = \reals^k$, $\mfv(\w) = \w$, $\ve(\w)$ is a convex function, and $\ve(\w) = \err(\w) \ \forall \w \in \W$.
\end{theoremN}
\begin{proof}
The proof sketch in the main text showed that, under the assumptions given in the theorem statement, minimizing \eqref{eq:bilevel_linear} is equivalent to solving the optimization problem:
\begin{equation} \label{eq:min_validation}
        \min_{x \in \set{0, 1}^n} \set { \ve(x) } \ee
\end{equation}
It remains to show that this problem is NP-hard when $\ve$ is an arbitrary
convex function.

We show this using a reduction from {\sc 0/1 Integer Programming}.
In an instance of {\sc 0/1 Integer Programming}, we are given as input a
vector $c \in \reals^n$, a matrix $A \in \reals^{m \times n}$, and a vector $b \in \reals^{m}$, and we wish to solve the optimization problem:
\begin{equation} \label {eq:zeroone}
        \min_{x \in \set{0, 1}^n} \set { c \cdot x  } \mathrm{\ subject\ to\ } Ax \le b \ee
\end{equation}
To reduce \eqref{eq:zeroone} to \eqref{eq:min_validation},
let $\ve(x) = c \cdot x + \ind(x)$, where $\ind(x)$
is an indicator for whether the constraint $A x \le b$ is satisfied:
\[
        \ind(x) = \begin{cases}
          0 & \text{ if } Ax \le b \\
          \infty & \text{otherwise.}
        \end{cases}
\]
It can be easily verified that $\ind(x)$ is a convex function, and
therefore $\ve(x)$ is convex.
By construction, minimizing $\ve(x)$ is equivalent to minimizing
$c \cdot x$ subject to $A x \le b$.
Minimizing \eqref{eq:zeroone} is therefore equivalent to minimizing \eqref{eq:min_validation},
and minimizing \eqref{eq:min_validation} is therefore NP-hard.
\end{proof}

\section{Details of Experiments}

We now provide additional details of our experiments that were omitted
from the main text.

\subsection{Methods}

We split the each of the four public datasets into training, validation, and
test sets, whose sizes are given in Table~\ref{tab:splits}.

For two of the datasets (\emph{caltech101} and \emph{oxford\_iiit\_pet}),
the original dataset was already divided into training
and test images.  In those cases we used the original training set as our
training set, and split the original test set randomly into validation and test
sets.  For the remaining two datasets (\emph{colorectal\_histology} and \emph{tf\_flowers} ),
we split the entire dataset randomly
into training, validation, and test sets.

\begin{table*}[h]
\caption{Sizes of training, validation, and test sets.}
\label{tab:splits}
\centering
  \begin{small}
  \begin{sc}
  \begin{tabular}{p{3.8cm}p{2cm}p{2cm}p{2.1cm}p{1.9cm}}
    \toprule
    Dataset & \#Training & \#Validation & \#Test \\
    \midrule

Caltech 101 & 3060 & 3371 & 3370 \\

Colorectal histology & 1667 & 1667 & 1666 \\

Oxford IIIT Pet & 3680 & 1835 & 1834 \\

TF-Flowers & 1670 & 1000 & 1000 \\

\bottomrule
  \end{tabular}
  \end{sc}
  \end{small}
\end{table*}

\subsection{Hyperparameter Tuning}

As discussed in the main text, our experiments tune the hyperparameter of
a loss function of the form:
\[
  \l(\w) = \logloss(\w) + \lambda_1 \norm{\w}_1 + \lambda_2 \norm{\w}_2^2 + \lambda_3 \luniform(\w) + \lambda_4 \ldropout(\w)
\]

The feasible set $\feasible$ is a hypercube, defined by the constraints
$\lambda_1 \in [.1, 100]$, $\lambda_2 \in [.1, 100]$, $\lambda_3 \in [0, .1]$,
and $\lambda_4 \in [0, 1]$.
For each hyperparameter, the feasible range
was determined by performing a one-dimensional grid search on one of the datasets (\emph{tf\_flowers}), and choosing a feasible range that included all the values that appeared to
have a chance of performing well.

To apply \ref{alg:lll} to this problem, we must formally define an additional
hyperparameter $\lambda_5$, which acts as a
multiplier on $\logloss(\w)$ and whose value is constrained to be 1.
When then use \ref{alg:lll} to learn a linear loss function of the form
$\l(\w) = \lambda \cdot \mfv(\w)$, where $\mfv(\w) = \tup{ \norm{\w}_1, \norm{\w}_2^2, \luniform(\w), \ldropout(\w), \logloss(\w)}$.

\newcommand{\squaredlength}[1]{\|#1\|^2}
\newcommand{\frobsq}[1]{\|#1\|^2_F}

The \ref{alg:lll} takes as input a hyperparameter $\epsilon \ge 0$.  We did not
tune this hyperparameter, but instead set its value in a heuristic way so
as to approximately equalize the contributions of the two penalty terms
in \eqref{eq:cost} (the loss-matching and gradient-matching penalties).
We achieve this by setting $\epsilon = \frac {\sum_{i=1}^m \squaredlength{\llgradi{i}}} {\sum_{i=1}^m \frobsq{\lljacobi{i}}}$, where for any matrix
$J$, we use $\frobsq{J}$ to denote the squared Frobenius norm.

When training using AdaGrad, we use a batch size of 1 and a learning rate multiplier of .1.

\subsection{Learning Novel Regularizers Online}

In these experiments, we used \ref{alg:tl} to learn a
problem-specific regularizer $r \in \conv(X)$.
The set $X$ was of size 50, and was obtained by looking at the model
weights after training for one epoch.  Specifically, $X = \set{x_1, x_2, \ldots, x_{50}}$, where the $x_i$ values are indexed in ascending order, and are chosen so that $\frac {n} {49} \pm 1$ weights fall into each interval $[x_i, x_{i+1}]$, where $n$ is the total number of weights.

The \ref{alg:tl} algorithm requires as input a set of initial models, $\W_0$.  In these experiments, we obtain these models by running AdaGrad for one epoch at a time, stopping at the end of a certain epoch $t$.  This is equivalent to using AdaGrad for the first $t$ epochs of training, then switching to AdaGrad + \ref{alg:tl} for epochs $t+1$ onward (in Figure~\ref{fig:learned_regularizers}, $t=4$, which is why the first four data points are identical for both algorithms).  We choose $t$ by evaluating validation loss at the end of each epoch, and stopping as soon as we observe validation loss worse than at the end of the previous epoch (i.e., as soon as AdaGrad first begins to overfit the validation set).

\end{appendices}

\end{document}